\documentclass[3p,twocolumn]{elsarticle}
\usepackage{generic}

\usepackage{algorithmic}
\usepackage{textcomp}

\usepackage{import}

\usepackage{verbatim}		
\usepackage{amsmath}
\usepackage{amssymb}

\usepackage{amsthm}
\usepackage{mathtools}	
\usepackage{grffile}	
\usepackage[tight,footnotesize]{subfigure}
\usepackage{microtype} 
\usepackage{color}
\usepackage{url}
\usepackage[ruled,vlined,linesnumbered]{algorithm2e}
\usepackage{bm} 
\usepackage{comment}
\usepackage{arydshln}

\usepackage{enumitem}
\usepackage{booktabs}

\usepackage{adjustbox}
\usepackage{tikz}
\usetikzlibrary{shadings, patterns, angles, quotes, arrows.meta, shapes, decorations.pathmorphing, decorations.shapes, decorations.text, positioning}
\usetikzlibrary{calc,intersections,arrows.meta}
\usepackage{pgfplots}

\usepackage{hyperref}
\hypersetup{
	colorlinks=true,
	linkcolor=black,
	citecolor=black,
}

	\tikzset{
	pil/.style={
		->,
		thick,
		shorten <=2pt,
		shorten >=2pt,}
}

\theoremstyle{remark}
\newtheorem{remarkx}{Remark}
\newenvironment{remark}
{\pushQED{\qed}\remarkx}
{\popQED\endremarkx}

\theoremstyle{definition}
\newtheorem{defn}{Definition}
\newtheorem{assump}{Assumption}
\newtheorem{problem}{Problem}
\newtheorem*{problem*}{Problem}

\theoremstyle{plain}
\newtheorem{theorem}{Theorem}
\newtheorem{lemma}{Lemma}

\newtheorem{prop}{Proposition}

\newcommand{\cmmnt}[1]{}

\newcommand{\scalemath}[2]{\scalebox{#1}{\mbox{\ensuremath{\displaystyle #2}}}}

\begin{document}


\begin{frontmatter}

\title{Singularity-Free Guiding Vector Fields on $\mathrm{SO}(3)$ with Designer-Specified Progression Behavior\tnoteref{t1}}

\tnotetext[t1]{The work of H.G.\ de Marina is supported by the \emph{Ram\'on y Cajal} grant RYC2020-030090-I from the Spanish Ministry of Science and by the ERC Starting Grant \emph{iSwarm} 101076091.}

\author[1]{Jes\'us Bautista Villar}
\ead{jesusbv@ugr.es}

\author[1]{Hector Garcia de Marina}
\ead{hgdemarina@ugr.es}

\affiliation[1]{organization={Department of Computer Engineering, Automation and Robotics, and Institute of Mathematics (IMAG), University of Granada},
                city={Granada},
                country={Spain}}


\begin{abstract}
This paper develops a singularity-free guiding vector field (SF-GVF) for path following on the special orthogonal group $\mathrm{SO}(3)$. First, we lift the Euclidean SF-GVF construction to $\mathrm{SO}(3)$, integrating the augmented-state approach with the intrinsic Lie-group geometry and obtaining a closed-form geometric guidance law whose integral curves converge to a designer-specified attitude path. The field is defined on a dense open subset of $\mathrm{SO}(3)$, excluding only the measure-zero antipodal set — a manifestation of the topological obstruction to continuous global stabilization on $\mathrm{SO}(3)$. The construction requires no per-step optimization and produces a control input intrinsically in $\mathfrak{so}(3)$ as body angular rates. Second, we formalize the \emph{progression behavior} along the path as a designer-supplied function $\nu(\xi)$, promoting the parametric speed from an implicitly resolved degree of freedom to a first-class design specification. In contrast to the Euclidean condition $v = 0$, which excludes vehicles with minimum-speed constraints, the corresponding condition $\omega = 0$ on $\mathrm{SO}(3)$ is physically admissible for most platforms with active attitude control, making the progression behavior a design freedom structurally available on $\mathrm{SO}(3)$ but absent in the Euclidean setting. The framework's structural results are established under a bi-invariant Riemannian metric and hold uniformly across choices of path, progression, and Lyapunov gain. The framework is illustrated in simulation on self-intersecting paths under both constant and point-convergence progression behaviors.
\end{abstract}
\end{frontmatter}



\section{Introduction} 
\label{sec: intro}

Path following on $\mathrm{SO}(3)$ --- guiding a rigid body's attitude toward and along a designer-specified reference curve, rather than tracking a time-parameterized trajectory --- arises naturally in aerospace, robotics, and space applications. Many such tasks do not require exact timing to be executed successfully: what matters is that the geometry is executed accurately, not that it is completed at a precise instant. 
A fixed-wing aircraft executing a loop, a spacecraft scanning a celestial region, a quadrotor performing an aerobatic figure, or a robotic end-effector following an orientation curve can all be assigned a strict time schedule, but in practice it is often immaterial whether the maneuver takes ten seconds or twelve, provided the geometry is right. 
Whenever exact timing is not part of the specification, path following serves as the natural framework --- not as a replacement for trajectory tracking, but as the appropriate alternative --- and motivates a distinct body of methodology grounded in the decoupling of geometry from temporality \cite{aguiar2008performance, kapitanyuk2017guiding}. However, the rate at which the path is traversed — the \emph{progression} — is typically left implicit, fixed by the guidance law or by the platform's kinematic envelope; one of the contributions of this paper is to treat it explicitly as a design specification on $\mathrm{SO}(3)$.

Among path-following approaches, guiding vector field (GVF) methods are particularly attractive: they provide a closed-form geometric guidance law whose integral curves converge to the desired path, without temporal commitment \cite{goncalves2010vector, kapitanyuk2017guiding, rezende2021constructive}. Extending GVFs from Euclidean space to $\mathrm{SO}(3)$ is, however, non-trivial. The compactness of $\mathrm{SO}(3)$, the topological obstruction to continuous global stabilization \cite{bhat2000topological}, and the absence of a global chart all conspire against direct generalization. Recent extensions to manifolds \cite{yao2023manifolds} and matrix Lie groups \cite{vinicius2026liegroups} have made progress, but neither incorporates the singularity-free augmentation of the singularity-free guiding vector field (SF-GVF) in \cite{yao2021singularity}, which combines a parametric description of the path with the lifting of the path parameter to circumvent topological obstructions in the Euclidean setting.

The first contribution of this paper is the construction of a singularity-free guiding vector field on $\mathrm{SO}(3)$, integrating the augmented-state approach of \cite{yao2021singularity} with the intrinsic geometry of the Lie group. The construction yields global convergence on the open chart $\mathcal{D} = \{\xi : \|\phi\| < \pi\}$, the natural domain of the logarithmic error coordinate, without per-step optimization, and with the control input intrinsically in $\mathfrak{so}(3)$ body-rates. The chart restriction $\|\phi\| < \pi$ is a direct manifestation of the topological obstruction \cite{bhat2000topological} and is the price of smooth time-invariant feedback; the construction is compatible with synergistic hybrid switching \cite{mayhew2011quaternion, mayhew2013synergistic, berkane2017hybrid, wang2021hybrid} for designers who require globally defined control laws.

The second contribution concerns the \emph{progression behavior} $\nu$, the rate at which the system advances along the reference path. A first-order GVF commands the platform's velocity: linear velocity $v$ in the Euclidean setting, body angular velocity $\omega$ on $\mathrm{SO}(3)$. The designer has freedom over the progression rate along the path in both settings; the \emph{scope} of that freedom depends on whether the commanded velocity can vanish while the platform remains physically feasible. For Euclidean SF-GVF, the condition $v=0$ is incompatible with vehicles that must maintain motion (fixed-wing aircraft, ground vehicles, marine surface vessels), restricting the designer to strictly positive progression rates. This excludes behaviors that require the system to stop, such as converging to and holding a specific configuration on the path. On $\mathrm{SO}(3)$, the corresponding condition $\omega = 0$ is admitted by most platforms with active attitude control, with spin-stabilized vehicles as a notable exception. The admissible design space therefore expands: the progression behavior $\nu$ can be specified as an arbitrary function of the state, enabling behaviors unavailable in the Euclidean case, including driving the system to a specific target attitude and holding it, synchronizing multiple agents at a common orientation, or throttling progression through demanding maneuvers. We characterize the conditions $\nu$ must satisfy and show that the framework's structural results hold for any smooth choice.

Our framework exposes four design elements: the parametric reference $R_\mathrm{ref}(\gamma)$ describing the path, the progression behavior $\nu(\xi)$ in closed loop with the system state $\xi$, the Lyapunov gain $K$, and the underlying Riemannian metric $M$. In this paper, however, the analysis is intentionally restricted to the bi-invariant metric $M=I$, which admits a closed-form logarithmic error and provides a particularly transparent analytical setting. Under this choice, the proposed guidance law satisfies the structural guarantees of convergence, invariance, and non-degeneracy for arbitrary choices of $R_\mathrm{ref}$, $K$, and admissible $\nu$.

The formulation nevertheless naturally accommodates general left-invariant metrics. Choosing $M\neq I$ provides a mechanism for shaping the convergence geometry, allowing the guidance law to favor platform-dependent criteria, such as control effort or actuation authority, instead of the minimum-angle geodesic induced by $M=I$. Such an extension lies beyond the scope of the present paper, as the structural analysis must be re-established once the geodesic structure and the Jacobian of the logarithmic error become metric dependent.

Correspondingly, the operational analysis assumes that the platform can execute the minimum-angle geodesic convergence induced by $M=I$, or that the initial attitude lies sufficiently close to the reference path; convergence transients incompatible with the platform's actuation envelope are thus ruled out by assumption rather than by construction. Within this scope, the proposed SF-GVF should be viewed as a \textit{guidance layer}. The designer specifies a reference path together with a progression behavior that are compatible with the intended platform, while the guiding vector field generates the corresponding body-rate command $\omega$. The standard guidance--control architecture in aerospace and robotics \cite{beard2012small, lee2010geometric, kapitanyuk2017guiding} separates this geometric guidance problem from the realization of the commanded rates through an inner-loop controller. Accordingly, the proposed framework is agnostic to the low-level implementation and can be combined with geometric controllers \cite{lee2010geometric}, incremental nonlinear dynamic inversion (INDI) \cite{smeur2016adaptive}, model predictive control (MPC) \cite{sun2022comparative}, or any controller capable of accurately tracking body angular velocities. This separation allows the same guidance law to be employed across platforms with markedly different dynamics and actuation mechanisms, provided the prescribed path and progression/convergence behavior are feasible.

\subsection{Related works}

\textbf{Trajectory tracking on $\mathrm{SO}(3)$.} A substantial body of work addresses trajectory tracking on $\mathrm{SO}(3)$ via Lyapunov methods on the Lie group \cite{bullo1999tracking, maithripala2006almost, lee2010geometric, maithripala2015intrinsic}. These works track a time-parameterized reference $R_\mathrm{ref}(t)$ produced upstream, using the geometric error $\phi = \log(R_\mathrm{ref}^\top R)^\vee$ and analogous Lyapunov constructions. The SF-GVF shares this geometric machinery but addresses a structurally different problem: trajectory tracking prescribes the time parameterization in open loop, whereas path following resolves the progression in closed loop with the state $\xi$ through the parametric coordinate $\gamma(\xi)$.

\textbf{Geometric extensions of GVFs.} \cite{yao2023manifolds} extended GVFs to Riemannian manifolds via zero-level-set characterization, with an attendant topological impossibility result for simple-closed paths. \cite{vinicius2026liegroups} extended the parametric construction of \cite{rezende2021constructive} to matrix Lie groups via a distance-function approach requiring per-step optimization. The present work occupies a distinct cell in this design space: a Lie-group setting via the singularity-free augmentation of \cite{yao2021singularity}, yielding global convergence on $\mathcal{D}$ without per-step optimization, with the control input intrinsically in $\mathfrak{so}(3) \cong \mathbb{R}^3$ rather than in the higher-dimensional embedding space of $\mathrm{SO}(3) \subset \mathbb{R}^{3\times 3}$.

\textbf{Progression behavior as a design variable.} The closest prior recognition of $\dot\gamma$ as a design variable is \cite{yao2022coord}, who use it as a consensus variable for multi-agent path-following coordination in Euclidean space. Our progression behavior $\nu(\xi)$, which specifies $\dot\gamma$ in closed loop with the state, generalizes this insight: consensus-based coordination is one instance, alongside the operational behaviors discussed in the introduction.

\subsection{Outline}
Section~\ref{sec: preliminaries} reviews $\mathrm{SO}(3)$, Riemannian metrics on Lie groups, and present the $\mathrm{SO}(3)$ path-following problem. Section~\ref{sec: sfgvf} constructs the SF-GVF and establishes convergence, invariance, and non-degeneracy. Section~\ref{sec: design} develops the construction of $R_\mathrm{ref}$, comparing three parametrizations. Section~\ref{sec: simulations} illustrates the framework via simulations on representative operational scenarios. Section~\ref{sec: conclusions} concludes.


\section{Preliminaries and Problem Statement}
\label{sec: preliminaries}

For a vector $u \in \mathbb{R}^3$, $\|u\|$ denotes its Euclidean norm. The operator $\times$ denotes the vector cross product, while dot product will be expressed as $u^\top v$, with $u,v\in\mathbb{R}^3$. $I$ is the $3 \times 3$ identity matrix.

\subsection{The Special Orthogonal Group $\mathrm{SO}(3)$}

The attitude of a rigid body in 3D space is fully described by three mutually orthogonal unit vectors, encoding how the body frame is oriented relative to a reference frame. Collecting these vectors as columns/rows of a matrix yields the most geometrically explicit representation of attitude: the \emph{rotation matrix}. The set of all such matrices forms the \emph{special orthogonal group}
\begin{equation}
\label{eq: SO3}
    \scalemath{0.95}{
    \mathrm{SO}(3) := \{R \in \mathbb{R}^{3 \times 3} \,:\, R^\top R = I,\, \det(R) = 1\}
    },
\end{equation}
which encodes all proper rotations in $\mathbb{R}^3$. Since $\mathrm{SO}(3)$ is simultaneously a smooth manifold and a group under matrix multiplication, it is a \emph{Lie group}.

Differentiating the orthogonality constraint $R^\top R = I$ shows that $R^\top \dot R$ is skew-symmetric,  so $\dot R = R S$ for some skew-symmetric $S \in \mathbb{R}^{3\times3}$. Thus the tangent space at $R$ of the $\mathrm{SO}(3)$ manifold is given by
\begin{equation*}
    \mathrm{T}_R\mathrm{SO}(3) := \{R S \, : \, S = - S^\top,\; S \in \mathbb{R}^{3\times 3}\}.
\end{equation*}

The \emph{Lie algebra} associated to $\mathrm{SO}(3)$ is the tangent space at the identity $\mathrm{T}_I\mathrm{SO}(3)$, i.e., the set of skew-symmetric matrices
\begin{equation} \label{eq: SO3_algebra}
    \mathfrak{so}(3) = \{ S \in \mathbb{R}^{3 \times 3} \,:\, S = - S^\top \},
\end{equation}
additionally equipped with the \emph{Lie bracket} $[\cdot,\cdot] : \mathfrak{so}(3) \times \mathfrak{so}(3) \rightarrow \mathfrak{so}(3)$, given by the matrix commutator $$[S_1,S_2] := S_1 S_2 - S_2 S_1.$$
Since every tangent space is obtained from $\mathrm{T}_I\mathrm{SO}(3)$ by left multiplication, i.e., $\mathrm{T}_R\mathrm{SO}(3) = R \,\mathfrak{so}(3)$, tangent vectors $\dot R \in \mathrm{T}_R\mathrm{SO}(3)$ may be identified with $\mathfrak{so}(3)$ elements through $S = R^\top \dot R$. This identification is known as the \emph{left-trivialization} and will be adopted throughout the paper for both kinematics and Riemannian metrics.

The space $\mathfrak{so}(3)$ is three-dimensional and isomorphic to $\mathbb{R}^3$ as a vector space. The \emph{hat map} ${}^\wedge : \mathbb{R}^3 \rightarrow \mathfrak{so}(3)$ provides the explicit $\mathfrak{so}(3)\cong\mathbb{R}^3$ isomorphism
\begin{equation*} \label{eq: hat_map}
    \tau^\wedge = \begin{bmatrix} 
    0 & -\tau_{z} & \tau_{y} \\ 
    \tau_{z} & 0 & -\tau_{x} \\ 
    -\tau_{y} & \tau_{x} & 0 
    \end{bmatrix}\in \mathfrak{so}(3),
    \quad
    \tau = \begin{bmatrix}\tau_x \\ \tau_y \\ \tau_z \end{bmatrix},
\end{equation*}
with inverse \emph{vee map} ${}^\vee : \mathfrak{so}(3) \rightarrow \mathbb{R}^3$, satisfying $(\tau^\wedge)^\vee = \tau$. Under this identification, $\tau^\wedge u = \tau \times u$ for any $u \in \mathbb{R}^3$, and the Lie bracket corresponds to the cross product, i.e., $[\tau_1^\wedge, \tau_2^\wedge] = (\tau_1 \times \tau_2)^\wedge$.

The \emph{adjoint map} $\mathrm{Ad}_R : \mathfrak{so}(3) \rightarrow \mathfrak{so}(3)$ describes how Lie algebra elements transform under a change of reference configuration on the group by conjugation:
\begin{equation} \label{eq: adjoint_map}
    \mathrm{Ad}_R(S) := R S R^\top.
\end{equation}
The infinitesimal counterpart of $\mathrm{Ad}_R$, obtained by differentiating at the identity, is the \emph{adjoint operator} $\mathrm{ad}_{S} : \mathfrak{so}(3) \rightarrow \mathfrak{so}(3)$:
\begin{equation*}
    \mathrm{ad}_{S_1}(S_2) := [S_1, S_2], \quad S_1, S_2 \in \mathfrak{so}(3),
\end{equation*}
which coincides with the Lie bracket. Under the $\mathfrak{so}(3)\cong\mathbb{R}^3$ isomorphism, \eqref{eq: adjoint_map} reduces to the standard vector rotation, i.e.
\begin{equation} \label{eq: adjoint_id}
    (\mathrm{Ad}_R(\tau^\wedge))^\vee = R\tau,
\end{equation}
and $\mathrm{ad}_{\tau_1^\wedge}(\tau_2^\wedge) = (\tau_1 \times \tau_2)^\wedge$.

Every element of $\mathfrak{so}(3)$ can be mapped to $\mathrm{SO}(3)$ via the \emph{exponential map} $\exp : \mathfrak{so}(3) \rightarrow \mathrm{SO}(3)$. Given $\tau \in \mathbb{R}^3$ with $\theta = \|\tau\|$, Rodrigues' formula gives
\begin{equation} \label{eq: exp}
    \exp(\tau^\wedge) = I + \frac{\sin\theta}{\theta}\,\tau^\wedge + \frac{1-\cos\theta}{\theta^2}\,(\tau^\wedge)^2.
\end{equation}
The inverse, the \emph{logarithmic map} $\log : \mathrm{SO}(3) \rightarrow \mathfrak{so}(3)$, recovers this element of the Lie algebra as
\begin{equation} \label{eq: log}
    \log(R) =
    \begin{cases}
        0_{3\times 3}, & R = I,\\[4pt]
        \dfrac{\theta}{2\sin\theta}(R - R^\top), & R \neq I,
    \end{cases}
\end{equation}
where $\theta = \arccos((\mathrm{tr}(R)-1)/2)$, so that $\|\log(R)^\vee\| = \theta$. The map $\log(R)$ is uniquely defined for $\theta \in [0,\pi)$ and becomes singular at $\theta = \pi$.

The exponential and logarithmic maps are smooth, and their differentials capture how infinitesimal changes in the Lie algebra $\delta \tau^\wedge \in \mathfrak{so}(3)$ relate to infinitesimal changes on the group $\delta R \in T_R\mathrm{SO}(3)$, with $R = \exp(\tau^\wedge)$. The \emph{right Jacobian} $J_r(\tau) : \mathbb{R}^3 \rightarrow \mathbb{R}^{3\times3}$ is defined as the linear map satisfying
\begin{equation} \label{eq: exp_dif}
\delta R = R \, (J_r(\tau)\delta\tau)^\wedge,
\end{equation}
and whose inverse $J_r^{-1}(\tau)$ encodes the differential of the logarithmic map, satisfying 
\begin{equation} \label{eq: log_dif}
\delta\tau = J_r^{-1}(\log(R)^\vee) \, (R^\top \delta R)^\vee.
\end{equation} 
To compute $J_r(\tau)$, one perturbs $\exp((\tau + \epsilon\delta\tau)^\wedge)$ and expand to first order in $\epsilon$ using the Baker-Campbell-Hausdorff formula. Collecting the first-order terms yields the series representation
\begin{equation} \label{eq: Jr_series}
    J_r(\tau) = \sum_{k=0}^{\infty} \frac{(-1)^k}{(k+1)!}\,\mathrm{ad}_{\tau^\wedge}^k,
\end{equation}
where $\mathrm{ad}_{\tau^\wedge}^k$ denotes the $k$-fold composition of $\mathrm{ad}_{\tau^\wedge}$.
For $\mathfrak{so}(3)$, the Cayley-Hamilton theorem guarantees that this series closes in finite form:
\begin{equation} \label{eq: Jr_closed}
    J_r(\tau) = I - \frac{1-\cos\theta}{\theta^2}\,\tau^\wedge + \frac{\theta-\sin\theta}{\theta^3}\,(\tau^\wedge)^2,
\end{equation}
where $\theta = \|\tau\|$, with inverse
\begin{equation} \label{eq: Jr_inv}
    J_r^{-1}(\tau) = I + \frac{1}{2}\,\tau^\wedge + \frac{1}{\theta^2}\!\left(1 - \frac{\theta}{2}\cot\frac{\theta}{2}\right)(\tau^\wedge)^2.
\end{equation}

For further background on $\mathrm{SO}(3)$ and Lie groups in control and robotics, we refer the reader to \cite{Bullo2005, so3_catalanes}.

\subsection{Riemannian metrics and geodesic distances on $\mathrm{SO}(3)$} 

Let $M\in\mathbb{R}^{3\times3}$ be symmetric positive definite. Through the
$\mathfrak{so}(3)\cong\mathbb{R}^3$ isomorphism, it defines an inner product on
$\mathfrak{so}(3) = T_I\mathrm{SO}(3)$,
\begin{equation}\label{eq: so3_inner_product}
    \langle S_1, S_2\rangle_I^M := (S_1^\vee)^\top M\, S_2^\vee,
    \quad S_1,S_2\in\mathfrak{so}(3).
\end{equation}

Left-translation extends \eqref{eq: so3_inner_product} to a Riemannian metric
on $\mathrm{SO}(3)$: for $X,Y\in T_R\mathrm{SO}(3)$,
\begin{equation}\label{eq: left_invariant_metric_M}
    \langle X, Y\rangle_R^M := \langle R^\top X,\, R^\top Y\rangle_I^M,
\end{equation}
with induced norm $\|X\|_R^M := \sqrt{\langle X, X\rangle_R^M}$. The
left-trivialization $X\mapsto R^\top X$ maps $T_R\mathrm{SO}(3)$ isomorphically
onto $\mathfrak{so}(3)$, so \eqref{eq: left_invariant_metric_M} is well defined
and left-invariant by construction.

A left-invariant metric is also right-invariant (hence \emph{bi-invariant}) if and only if the inner product on the algebra is \emph{Ad-invariant}:
\begin{equation}\label{eq: ad_invariant}
\scalemath{0.9}{\langle \mathrm{Ad}_R S_1,\, \mathrm{Ad}_R S_2 \rangle_I^M = \langle S_1, S_2 \rangle_I^M, \quad \forall R \in \mathrm{SO}(3)}.
\end{equation}
On $\mathfrak{so}(3)$ this forces $M \propto I$.

The induced \emph{geodesic distance} between $R_1,R_2 \in \mathrm{SO}(3)$ is
\begin{equation}\label{eq: so3_metric_geod_G}
d_{\mathrm{SO}(3)}^M(R_1, R_2) := \inf_{R(\cdot)} \int_0^1 \|\dot R(t)\|_{R(t)}^M \, \mathrm{d}t,
\end{equation}
the infimum over geodesics $R:[0,1]\rightarrow\mathrm{SO}(3)$ with $R(0) = R_1$ to $R(1) = R_2$. For general $M \succ 0$, geodesics satisfy Euler's rigid-body equation $M\dot\omega = (M\omega) \times \omega$ in body coordinates and admit no closed-form expression \cite[Chapter 4.4]{Bullo2005}. However, for the $M=I$ bi-invariant metric, geodesics through the identity are one-parameter subgroups $t \rightarrow \exp(tS)$, the Riemannian and matrix exponentials coincide, and \eqref{eq: so3_metric_geod_G} simplifies to
\begin{equation}\label{eq: so3_metric_geod}
d_{\mathrm{SO}(3)}(R_1, R_2) = \bigl\|\log(R_1^\top R_2)^\vee\bigr\|.
\end{equation}
the minimum rotation angle bringing $R_1$ to $R_2$.

\subsection{Attitude path specification} 

The attitude of a robot in 3D space is described by a rotation matrix
$R \in \mathrm{SO}(3)$ expressed in an inertial frame, with columns
$R = \begin{bmatrix} b_x & b_y & b_z \end{bmatrix}$ the orthonormal
body-fixed axes $b_x, b_y, b_z \in \mathbb{S}^2$.
Denoting the body angular velocity by $\omega \in \mathbb{R}^3$ and
$\omega^\wedge \in \mathfrak{so}(3)$, the attitude evolves under the
left-invariant kinematics
\begin{equation}\label{eq: Rdot}
    \dot{R} = R\omega^\wedge \in T_R\,\mathrm{SO}(3).
\end{equation}
Equivalently, in the inertial frame $\dot R = \omega_E^\wedge R$ with
$\omega_E = R\omega = (\mathrm{Ad}_R(\omega^\wedge))^\vee$, so $\mathrm{Ad}_R$ maps the
angular velocity from the body to the inertial frame.

The path-following task involves two design choices: \emph{which} attitude
path to follow and, since there is no temporal constraint, \emph{how} to traverse it. The path is specified by a reference curve $R_{\mathrm{ref}} : \mathbb{R} \to \mathrm{SO}(3)$, $\gamma \mapsto R_{\mathrm{ref}}(\gamma)$, parametrized by a scalar path variable $\gamma \in \mathbb{R}$ and chosen by the designer. We define the \emph{reference body angular velocity}
\begin{equation}\label{eq: omega_ref}
\omega_{\mathrm{ref}}(\gamma) := \big(R_{\mathrm{ref}}^\top(\gamma)\, R_{\mathrm{ref}}'(\gamma)\big)^\vee \in \mathbb{R}^3,
\end{equation}
where $(\cdot)' := \mathrm{d}/\mathrm{d}\gamma$, so that $R_{\mathrm{ref}}' = R_{\mathrm{ref}}\omega_{\mathrm{ref}}^\wedge$ mirrors \eqref{eq: Rdot}, with $\gamma$ replacing time. Along the reference curve, the chain rule then gives $\dot R_{\mathrm{ref}} = R_{\mathrm{ref}}\omega_{\mathrm{ref}}^\wedge \dot \gamma$.

\begin{assump}\label{assump: r_ref}
    The map $\gamma \mapsto R_{\mathrm{ref}}(\gamma)$ is smooth ($C^\infty$) and regular, i.e.\ $\omega_{\mathrm{ref}}(\gamma) \neq 0$ for all $\gamma \in \mathbb{R}$.
\end{assump}

Smoothness ensures existence and uniqueness of closed-loop solutions, while
regularity guarantees that $R_{\mathrm{ref}}$ is always in motion, so $\gamma$
is a valid progress variable with a well-defined nonzero tangent everywhere.

Deviations of the current attitude $R$ from the reference at parameter
$\gamma$ are quantified by the \emph{attitude error matrix}
$$
R_e(R, \gamma) := R_{\mathrm{ref}}(\gamma)^\top R \in \mathrm{SO}(3)
$$
and its
$\mathbb{R}^3$ image, the \emph{attitude error}
\begin{equation} \label{eq: phi}
    \phi(R, \gamma) := \log\big(R_e(R, \gamma)\big)^\vee \in \mathbb{R}^3.
\end{equation}
By \eqref{eq: so3_metric_geod}, $\|\phi(R,\gamma)\|
= d_{\mathrm{SO}(3)}\big(R_{\mathrm{ref}}(\gamma), R\big)$, so $\phi$ is the
geodesic error and vanishes if and only if $R = R_{\mathrm{ref}}(\gamma)$.

\begin{defn}[Reference attitude path]
    Given $R_{\mathrm{ref}}$ satisfying Assumption~\ref{assump: r_ref}, the
    \emph{reference attitude path} is
    \begin{equation} \label{eq: path_phy}
        \mathcal{P} := \big\{R \in \mathrm{SO}(3) \,:\,
        \exists\,\gamma,\ \phi(R, \gamma) = 0\big\},
    \end{equation}
    i.e.\ the image of the curve $R_{\mathrm{ref}}$.
\end{defn}

Since $\log$ is injective only away from rotations of angle $\pi$, the error
$\phi$ is smooth precisely on the open domain
\begin{equation}\label{eq: domain}
    \mathcal{D} := \big\{(R, \gamma) \in \mathrm{SO}(3) \times \mathbb{R}
    \,:\, \|\phi(R, \gamma)\| < \pi\big\},
\end{equation}
on which the geodesic distance is uniquely defined; all subsequent analysis is
restricted to $\mathcal{D}$.

\subsection{Progression behavior specification}
Reaching $\mathcal{P}$ is not sufficient, the attitude must \emph{traverse} it. The second design choice is therefore a \emph{progression behavior} $\nu : \mathcal{D} \to \mathbb{R}$, a map prescribing the rate at which $\gamma$ should advance along $R_{\mathrm{ref}}$. 

\begin{assump}\label{assump: nu}
The map $\nu: \mathcal D \to \mathbb{R}$ is continuously differentiable.
\end{assump}

\begin{remark}
The behavior $\nu$ encodes operational intent and can take diverse forms: constant traversal $\nu \equiv \nu_0$ for uniform motion along $R_{\mathrm{ref}}$, point-convergence $\nu(\xi) = -k(\gamma - \gamma^*)$ for driving the system to a target $R_{\mathrm{ref}}(\gamma^*)$, multi-agent synchronization via $\nu$ dependent on neighboring agents, or feasibility-aware throttling that slows progression through demanding sections of the path. The framework imposes no constraint on $\nu$ beyond Assumption \ref{assump: nu}.
\end{remark}

The pair $(R_{\mathrm{ref}}, \nu)$ thus encodes
the full specification: which path, and how it is traversed.

\subsection{Problem Statement} 

Following \cite{yao2021singularity}, we turn the path variable into a state to circumvent the topological obstruction that prevents global convergence on closed or self-intersecting paths via pure gradient methods.

Augmenting the kinematics \eqref{eq: Rdot} with $\dot\gamma = u_\gamma$:
\begin{equation} \label{eq: Rdot_ext}
    \dot R = R\omega^\wedge, \quad \dot \gamma = u_\gamma ,
\end{equation}
where $u_\gamma \in \mathbb{R}$ is a virtual control input. The augmented state
is $\xi := (R, \gamma) \in \mathrm{SO}(3) \times \mathbb{R}$, the input is
$\chi := [\omega^\top,\ u_\gamma]^\top \in \mathbb{R}^4$, and (writing
$\phi(\xi) = \phi(R,\gamma)$) the reference path $\mathcal{P}$ lifts to
\begin{equation} \label{eq: path_hgh}
    \scalemath{0.95}{
    \mathcal{P}^{\mathrm{hgh}} := \big\{\xi = (R, \gamma)
    \in \mathrm{SO}(3) \times \mathbb{R} \,:\, \phi(\xi) = 0\big\}
    }.
\end{equation}
As $\gamma$ is virtual, $\mathcal{P}^{\mathrm{hgh}}$ is not a physical attitude
set, but its projection onto $\mathrm{SO}(3)$ recovers $\mathcal{P}$; hence
convergence to $\mathcal{P}^{\mathrm{hgh}}$ implies convergence to
$\mathcal{P}$.

\begin{problem}[Path-following with progression behavior on $\mathrm{SO}(3)$] \label{prob: sfgvf}
    Given a reference curve $R_{\mathrm{ref}}$ satisfying
    Assumption~\ref{assump: r_ref}, with reference path
    $\mathcal{P} \subset \mathrm{SO}(3)$ and desired progression behavior $\nu$ satisfying Assumption~\ref{assump: nu}, design a
    continuously differentiable vector field $\chi : \mathcal{D} \to
    \mathbb{R}^4$, $\chi(\xi) = [\omega(\xi)^\top,\ u_\gamma(\xi)]^\top$,
    for the augmented system \eqref{eq: Rdot_ext} such that:
    \begin{enumerate}
        \item (Convergence) For all initial conditions $\xi_0 \in \mathcal{D}$, $\|\phi(\xi(t))\| \to 0$ as $t \to \infty$;
        \item (Path invariance) If $\xi_0 \in \mathcal{P}^{\mathrm{hgh}}$, then $\xi(t) \in \mathcal{P}^{\mathrm{hgh}}$ for all $t \geq 0$;
        \item (Non-degeneracy) $\chi(\xi) \neq 0$ for all $\xi \in \mathcal{D}$ with $\phi(\xi) \neq 0$ or
            $\nu(\xi) \neq 0$.
    \end{enumerate}
\end{problem}

\begin{remark}
    Conditions 2 and 3 of Problem \ref{prob: sfgvf} together govern the system's behavior on the path.
    Condition 2 ensures invariance: once $\phi = 0$, the attitude error remains zero. Condition 3 ensures structural well-posedness: zeros of $\chi$ correspond only to operational completion, when both the attitude error has converged ($\phi = 0$) and the behavior has come to rest ($\nu = 0$). If $\nu(\xi)\neq 0$ on $\mathcal{P}^{\mathrm{hgh}}$, condition 3 prevents the field from vanishing, ensuring that $\gamma$ continues to advance and the guiding point $R_{\mathrm{ref}}(\gamma(t))$ moves along $\mathcal{P}$.
\end{remark}


\section{SF-GVF for $\text{SO}(3)$ path following}
\label{sec: sfgvf}

\begin{figure}
    \centering
    \includegraphics[trim={0cm 0cm 0cm 0cm}, clip, width=1\columnwidth]{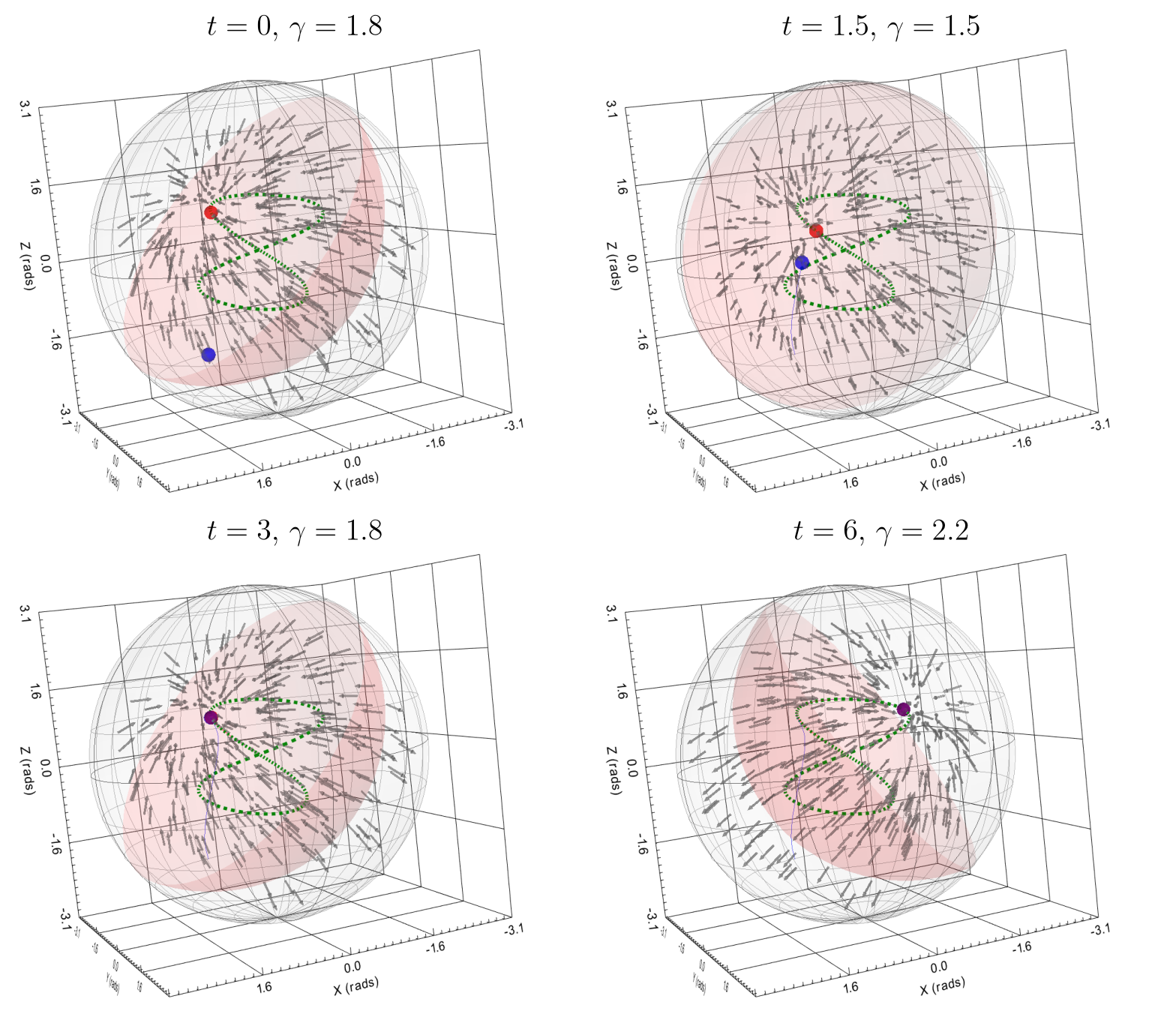}
    \caption{Visualization of the $\mathfrak{so}(3)$ component (the first three elements) of the $\mathrm{SO}(3)$ SF-GVF in \eqref{eq: gvf} for the Lissajous path (green dashed line) given by $\tau_\mathrm{ref}(\gamma) = [\sin(2\gamma),\, \sin(3\gamma),\, \sin(\gamma)]$, with $\nu = 0.1$ and $K=0.8I$. The $R$ exponential coordinates $\tau$ and reference curve coordinates $\tau_\mathrm{ref}(\gamma)$ are shown as blue and red dots, respectively. The grey ball is the $\mathbb{R}^3$ region where $\mathbb{R}^3 \cong \mathfrak{so}(3)$, while the red clipped ball is the singularity boundary where $\phi(\xi) = \pi$, i.e. $\xi \notin \mathcal{D}$.}
    \label{fig: gvf}
\end{figure}

To design a vector field on the input space that drives $\phi \rightarrow 0$, we first identify the linear map from $\chi$ to $\dot \phi$, which we call the \emph{error Jacobian}. Differentiating $\phi$ along \eqref{eq: Rdot_ext} and applying the differential of the logarithmic map \eqref{eq: log_dif}:
\begin{align*}
    \dot\phi &= J_r^{-1}(\phi) (R_e^\top \dot R_e)^\vee.
\end{align*}
The attitude error matrix evolves as 
$$\dot R_e = R_{\mathrm{ref}}^\top \dot R - R_{\mathrm{ref}}^\top \dot R_{\mathrm{ref}} (R_{\mathrm{ref}}^\top R) = R_e \omega^\wedge - \omega_{\mathrm{ref}}^\wedge \dot \gamma R_e.$$ 
Substituting and applying the adjoint identity \eqref{eq: adjoint_id}:
\begin{equation}\label{eq: phi_dot}
    \scalemath{0.95}{
    \dot\phi(\xi) = J_r^{-1}(\phi)\left[\omega - R_e(\xi)^\top \omega_\mathrm{ref}(\gamma) \dot\gamma\right] = J_\phi(\xi) \chi
    },
\end{equation}
with the \emph{error Jacobian} $J_\phi: \mathcal{D}\rightarrow\mathbb{R}^{3\times 4}$ given by
\begin{equation}\label{eq: J_phi}
    J_\phi(\xi) =
    J_r^{-1}(\phi)\Bigl[\,I \;\big|\; \underbrace{-R_e(\xi)^\top \omega_\mathrm{ref}(\gamma)}_{=: \,J_\gamma(\xi)}\,\Bigr].
\end{equation}
The matrix $J_\phi$ is the central object of our construction: it maps the input $\chi$ directly to the error rate $\dot\phi$. It is well-defined on $\mathcal{D}$ since $J_r^{-1}(\phi)$ is invertible for $\|\phi\|<\pi$, and depends on the path only through $R_\mathrm{ref}(\gamma)$ and $\omega_\mathrm{ref}$, which are quantities intrinsic to the curve on $\text{SO}(3)$.

The $\mathrm{SO}(3)$ SF-GVF we propose has the form
\begin{equation} \label{eq: gvf}
    \chi(\xi) =
    \underbrace{
    \begin{bmatrix}
    -J_\gamma(\xi) \\
    1
    \end{bmatrix}\nu(\xi)}_{\chi_t(\xi)}
    - J_\phi(\xi)^\top K \phi(\xi) \; \in \mathbb{R}^4,
\end{equation}
where $K \in \mathbb{R}^{3\times 3}$, $K\succ 0$, modulates the convergence rate.
The structure of \eqref{eq: gvf} mirrors the design of \cite{yao2021singularity} in Euclidean space: the tangential component $\chi_t$ lies in the kernel of $J_\phi$ and drives progress along the path, while the gradient-based component $- J_\phi^\top K \phi$ pulls the system toward $\mathcal{P}^\mathrm{hgh}$ (see \autoref{fig: gvf}). The precise sense in which these two components are decoupled is established by the following lemma.

\begin{lemma}[Tangential component] \label{lem: orth}
    The feedforward term $\chi_t(\xi)$ lies in the kernel of $J_\phi(\xi)$ for all $\xi \in \mathcal{D}$, i.e., $J_\phi(\xi)\chi_t(\xi) = 0$. Consequently, $\chi_t$ is orthogonal to the range of $J_\phi(\xi)^\top$, and the tangential and convergence terms in \eqref{eq: gvf} satisfy $\chi_t(\xi)^\top J_\phi(\xi)^\top K\phi(\xi) = 0$.
\end{lemma}
\begin{proof}
By definition of $J_\phi$ and $\chi_t$:
\begin{align*}
J_\phi(\xi)\chi_t(\xi) = J_r^{-1}(\phi)\bigl(-J_\gamma(\xi)\nu(\xi) + J_\gamma(\xi)\nu(\xi)\bigr) = 0,
\end{align*}
for all $\xi \in \mathcal{D}$. The orthogonality statement follows from $\chi_t^\top J_\phi^\top K\phi = (J_\phi\chi_t)^\top K\phi = 0$.
\end{proof}

\begin{remark}
    Lemma \ref{lem: orth} reveals that the feedforward angular velocity $-J_\gamma\nu = R_e^\top \omega_\mathrm{ref}\nu$ is the unique tangent direction (up to scaling by $\nu$) that places $\chi_t$ in $\ker\{J_\phi\}$. Geometrically, it transports the reference's body angular velocity into the current body frame via the adjoint of $R_e$, ensuring that following the reference at progression rate $\nu$ does not perturb the attitude error. Crucially, the orthogonality property holds independently of the gain matrix $K$: the kernel structure of $J_\phi$ is metric-independent. This is the $\mathrm{SO}(3)$ analogue in compact form of the orthogonality property in \cite[Lemma 1]{yao2021singularity}, where the propagation term is orthogonal to each gradient $\nabla\phi_i$ via the generalized cross product.
\end{remark}

\begin{prop}[Exponential convergence and path invariance] \label{prop: convergence}
    Under the SF-GVF \eqref{eq: gvf}, for every $\xi_0 \in \mathcal{D}$, the closed-loop trajectory $\xi(t)$ satisfies $\|\phi(\xi(t))\| \to 0$ exponentially as $t \to \infty$. Moreover, if $\xi_0 \in \mathcal{P}^{\mathrm{hgh}}$, then $\xi(t) \in \mathcal{P}^{\mathrm{hgh}}$ for all $t \geq 0$.
\end{prop}
\begin{proof}
Consider the $K$-weighted Lyapunov candidate $V_K(\xi) = \tfrac{1}{2}\phi(\xi)^\top K\,\phi(\xi)$. Along trajectories of \eqref{eq: Rdot_ext} with input $\chi = \chi_t - J_\phi^\top K\phi$, using \eqref{eq: phi_dot}:
\begin{align*}
\dot V_K &= \phi^\top K \dot\phi = \phi^\top K J_\phi \chi \\
&= \phi^\top K J_\phi \chi_t - \phi^\top K J_\phi J_\phi^\top K \phi.
\end{align*}
By Lemma~\ref{lem: orth}, $J_\phi\chi_t = 0$, hence $\phi^\top K J_\phi \chi_t = 0$, leaving
$$
\dot V_K = -\phi^\top K J_\phi J_\phi^\top K \phi = -\|J_\phi^\top K\phi\|^2.
$$
From \eqref{eq: J_phi}, $J_\phi J_\phi^\top = J_r^{-1}(I + J_\gamma J_\gamma^\top) J_r^{-\top} \succ 0$ on $\mathcal{D}$ (since $I + J_\gamma J_\gamma^\top \succ 0$ and $J_r^{-1}$ is invertible). Combined with $K \succ 0$, the matrix $K J_\phi J_\phi^\top K$ is symmetric positive definite. Let $\lambda_K(\xi) := \lambda_{\min}(K J_\phi J_\phi^\top K) > 0$. Then
$$
\dot V_K \leq -\lambda_K(\xi)\|\phi\|^2 \leq -\frac{2\lambda_K(\xi)}{\lambda_{\max}(K)} V_K,
$$
where the second inequality uses $V_K \leq \tfrac{1}{2}\lambda_{\max}(K)\|\phi\|^2$. This proves exponential convergence of $V_K$, and hence of $\|\phi\|$, to zero on $\mathcal{D}$. Invariance follows from $V_K(\xi_0) = 0 \Rightarrow V_K(\xi(t)) = 0$ for all $t \geq 0$.
\end{proof}

\begin{table*}
\centering
\caption{Comparison of path parametrizations for the SF-GVF interface.}
\label{tab: parametrizations}
\renewcommand{\arraystretch}{1.2}
\begin{tabular}{@{}lccccc@{}}
\toprule
\textbf{Parametrization} & \textbf{Specification} & \textbf{$R_\mathrm{ref}$ via} & \textbf{$\omega_\mathrm{ref}$ via} & \textbf{Singularity} & \textbf{Intuitive model} \\
\midrule
Euler angles (\S\ref{sec: path_euler}) 
    & $\eta(\gamma) \in \mathbb{R}^3$ 
    & Matrix product 
    & Eq.~\eqref{eq: omega_ref_euler} 
    & Gimbal lock 
    & Yaw/pitch/roll \\
Exp.\ coord. (\S\ref{sec: path_exp}) 
    & $\tau_\mathrm{ref}(\gamma) \in \mathbb{R}^3$ 
    & Single $\exp$ 
    & $J_r(\tau_\mathrm{ref})\dot\tau_\mathrm{ref}$ 
    & $\|\tau_\mathrm{ref}\|\!=\!\pi$ 
    & Axis + angle \\
Body rates (\S\ref{sec: path_rates}) 
    & $\omega_\mathrm{ref}(\gamma) \in \mathbb{R}^3$ 
    & Eq.~\eqref{eq: omega_to_R} 
    & Direct 
    & None 
    & Body rates \\
\bottomrule
\end{tabular}
\end{table*}

\begin{prop}[Non-degeneracy] \label{prop: no_sing}
    Under the SF-GVF \eqref{eq: gvf}, $\chi(\xi) = 0$ only if both $\phi(\xi) = 0$ and $\nu(\xi) = 0$.
\end{prop}
\begin{proof}
Suppose $\chi(\xi) = 0$, i.e., $\chi_t(\xi) = J_\phi(\xi)^\top K \phi(\xi)$ for some $\xi \in \mathcal{D}$. Taking the inner product with $\chi_t(\xi)$ and applying Lemma~\ref{lem: orth}:
$$
\|\chi_t(\xi)\|^2 = \chi_t(\xi)^\top J_\phi(\xi)^\top K\phi(\xi) = 0.
$$
Hence $\chi_t(\xi) = 0$, and since the fourth component of $\chi_t$ is $\nu(\xi)$, it implies that $\nu(\xi) = 0$. With $\chi_t = 0$, $\chi = 0$ reduces to $J_\phi^\top K \phi = 0$. Left-multiplying by $J_\phi$:
$$
J_\phi(\xi) J_\phi(\xi)^\top K \phi(\xi) = 0.
$$
Since $J_\phi J_\phi^\top \succ 0$ on $\mathcal{D}$, we have $K\phi(\xi) = 0$. Since $K \succ 0$ is invertible, $\phi(\xi) = 0$.
\end{proof}

\begin{theorem}\label{thm: gvf}
The SF-GVF \eqref{eq: gvf} solves Problem~\ref{prob: sfgvf}.
\end{theorem}
\begin{proof}
Conditions~1 and~2 follow from Proposition~\ref{prop: convergence}; Condition~3 follows from Proposition~\ref{prop: no_sing}.
\end{proof}

\begin{remark}[Conventional analysis to extensions of $K$]\label{rem: extensions}
With the bi-invariant metric ($M = I$) fixed, yielding the closed-form geodesic distance \eqref{eq: so3_metric_geod}, and the isomorphism $\mathfrak{so}(3) \cong \mathbb{R}^3$, the error coordinate $\phi \in \mathbb{R}^3$ and its Jacobian $J_\phi \in \mathbb{R}^{3\times 4}$ are Euclidean objects. Consequently, extensions of Proposition~\ref{prop: convergence} to richer gain choices as state-dependent $K(\xi)$, error-dependent $K(\phi)$, barrier-Lyapunov constructions keeping trajectories away from $\partial\mathcal{D}$ \cite{tee2009barrier}, or adaptive tuning \cite{krstic1995nonlinear}, reduce to standard nonlinear-control analyses on $\mathbb{R}^3$ (see, e.g., \cite[Chapter 4]{khalil} for Lyapunov techniques and \cite{slotine1991applied} for adaptive methods). The kernel structure of Lemma~\ref{lem: orth} and the non-degeneracy property of Proposition~\ref{prop: no_sing} carry over verbatim, as they depend only on $J_\phi$ being injective on $\mathcal{D}$ and $K\succ 0$. We do not develop these extensions here; the contribution of this paper is the geometric construction and the structural design of the SF-GVF, on top of which conventional control-theoretic tools apply directly.
\end{remark}

\section{Construction of the reference path}
\label{sec: design}

The $\mathrm{SO}(3)$ SF-GVF proposed in \eqref{eq: gvf} is parametrization-agnostic: it accepts any reference pair $(R_\mathrm{ref}(\gamma),\,\omega_\mathrm{ref}(\gamma))$ satisfying Assumption~\ref{assump: r_ref}. In this section, we present and analyze three specification strategies for constructing such pairs, ordered from the most intuitive to the most singularity-robust (see \autoref{tab: parametrizations}).

\subsection{Intrinsic Euler-angle composition}\label{sec: path_euler}
The most intuitive parametrization. The designer specifies three smooth functions $\eta(\gamma) = [\eta_1(\gamma),\, \eta_2(\gamma),\, \eta_3(\gamma)]^\top$ corresponding to rotations about body axes $a_1, a_2, a_3 \in \{x, y, z\}$, with the path constructed as
\begin{equation}\label{eq: euler_composition}
    \scalemath{0.95}{R_\mathrm{ref}(\gamma) = R_0 R_{a_1}(\eta_1(\gamma)) R_{a_2}(\eta_2(\gamma))  R_{a_3}(\eta_3(\gamma))},
\end{equation}
where $R_a(\cdot)$ denotes the elementary rotation about body axis $e_a$, and $R_0 \in \mathrm{SO}(3)$ is a fixed initial reference. Differentiating \eqref{eq: euler_composition} along the path and applying $\omega^\wedge = R^\top \dot R$ yields
\begin{equation} \label{eq: omega_ref_euler}
\begin{aligned}
\omega_\mathrm{ref}(\gamma) &= R_{a_3}^\top R_{a_2}^\top e_{a_1}\, \dot\eta_1(\gamma) \\
&+ R_{a_3}^\top e_{a_2}\, \dot\eta_2(\gamma) + e_{a_3}\, \dot\eta_3(\gamma).
\end{aligned}
\end{equation}
where each axis is transported back through the rotations to its right into the final body frame. For the standard aerospace Z-Y-X (yaw-pitch-roll) convention with $(a_1, a_2, a_3) = (z, y, x)$ and $\eta = (\psi, \theta, \phi)$:
\begin{equation*}\label{eq: omega_ref_zyx}
    \omega_\mathrm{ref}(\gamma) = R_x^\top R_y^\top e_z\, \dot\psi + R_x^\top e_y\, \dot\theta + e_x\, \dot\phi.
\end{equation*}

\begin{remark}[Gimbal lock and the forward map]
    Euler-angle composition exhibits gimbal lock at the second-angle configuration (e.g., $\theta = \pm\pi/2$ for Z-Y-X), where the linear map $T(\eta): \dot\eta \mapsto \omega$ in \eqref{eq: omega_ref_euler} drops to rank 2. This affects only the inverse maps $R \to \eta$ and $\omega \to \dot\eta$; the forward construction used here evaluates exactly at all configurations. Once $R_\mathrm{ref}, \omega_\mathrm{ref}$ are produced, the SF-GVF operates intrinsically on $\mathrm{SO}(3) \times \mathbb{R}$ without inverting back to the chart, so gimbal lock does not propagate to the controller. The interface's only limitation is expressive: at gimbal lock, $\omega_\mathrm{ref}$ is confined to the 2D image of $T$, restricting the body angular velocities specifiable through Euler rates at this instant.
\end{remark}

\subsection{Exponential-coordinate composition} \label{sec: path_exp}
For rotations specified by a single axis and angle, the construction $R_\mathrm{ref}(\gamma) = R_0 \exp(\tau_\mathrm{ref}(\gamma)^\wedge)$ with $\tau_\mathrm{ref}: \mathbb{R} \to \mathbb{R}^3$ gives a geometrically transparent representation: $\|\tau_\mathrm{ref}\|$ is the rotation angle and $\hat\tau_\mathrm{ref}$ the axis. The body angular velocity is $\omega_\mathrm{ref}(\gamma) = J_r(\tau_\mathrm{ref}(\gamma))\,\dot\tau_\mathrm{ref}(\gamma)$ via the right Jacobian \eqref{eq: Jr_closed}. The chart restriction $\|\tau_\mathrm{ref}\| < \pi$ excludes rotations through the antipodal locus; trajectories spanning larger angular displacements are constructed by multi-segment composition $R_\mathrm{ref}(\gamma) = R_0 \prod_k \exp(\tau_k(\gamma)^\wedge)$, with $\|\tau_k(\gamma)\| \ll \pi$, and $\omega_\mathrm{ref}$ obtained by the chain rule. Free of gimbal lock; subject only to the antipodal restriction within each segment.

\subsection{Body-rate specification.} \label{sec: path_rates}
Fully singularity-free. The designer specifies $\omega_\mathrm{ref}: \mathbb{R} \to \mathbb{R}^3$ smoothly together with an initial reference $R_0 \in \mathrm{SO}(3)$, and $R_\mathrm{ref}(\gamma)$ is obtained by integrating the rotation flow on the group:
\begin{equation}\label{eq: omega_to_R}
    \frac{dR_\mathrm{ref}}{d\gamma} = R_\mathrm{ref}(\gamma)\,\omega_\mathrm{ref}(\gamma)^\wedge, \quad R_\mathrm{ref}(0) = R_0.
\end{equation}
The flow is well-defined globally on $\mathrm{SO}(3)$ since the right-hand side is smooth and linear in $R$; no chart-based representation is ever invoked, and no singularities of any kind arise. 

Equation \eqref{eq: omega_to_R} can be integrated offline using a Lie-group integrator. For example, with step size $\Delta\gamma$,
$$
R_{k+1}=R_k\exp(\Delta\gamma\,\omega_\mathrm{ref}(\gamma_k)^\wedge),
$$
which preserves the orthogonality constraint exactly. Higher-order Lie-group schemes such as Runge-Kuta or Magnus integrators may also be employed \cite{hairer2006gni}. The resulting trajectory $R_\mathrm{ref}(\gamma)$ is stored as sampled or interpolated data for runtime use, while $\omega_\mathrm{ref}(\gamma)$ is retained directly from the original specification.


\section{Simulations}
\label{sec: simulations}

This section numerically validates the $\mathrm{SO}(3)$ SF-GVF developed in Section~\ref{sec: sfgvf} on two representative scenarios. The first is a minimal example intended to make the framework's behavior visually intuitive: a single-axis rotation with constant progression, illustrating convergence from an arbitrary initial attitude onto a designer-specified path and steady-state traversal at the prescribed rate. The second is a more demanding scenario that exercises the design freedom in $\nu$: a self-intersecting Lissajous curve on $\mathrm{SO}(3)$, along which the system is required to converge to and stop at a designer-specified point on the path

Numerical integration on $\mathrm{SO}(3)$ is performed with the \texttt{Lie++} library \cite{fornasier2023msceq, fornasier2025equivariant}, embedded into the \texttt{ssl\_simulator} Python framework\footnote{\url{github.com/Swarm-Systems-Lab/ssl_simulator}} for simulation and visualization. All code reproducing the results in this section is publicly available.\footnote{\url{github.com/Swarm-Systems-Lab/gvf_so3}}

\subsection{Single-axis rotation with constant progression}

The reference path is a single-axis rotation $R_\mathrm{ref}(\gamma) = \exp\left([0,\, -\gamma,\, 0]^\wedge\right)$, generating a continuous pitch-axis loop parameterized by $\gamma$. The SF-GVF is applied with Lyapunov gain $K = 0.8I$ and constant progression behavior $\nu = 0.5\ \mathrm{rad/s}$. The initial attitude is $R(0) = \exp\left([-\pi/4,\, -\pi/3,\, \pi/3]^\wedge\right)$, deliberately far from the path.

Figure~\ref{fig: sim1} shows the resulting trajectory. Two features of the SF-GVF are visible in this run. First, the geometric convergence: the attitude error $\|\phi\|$ decreases monotonically from its initial value, and the system reaches the path along the geodesic (under the bi-invariant metric $M = I$) rather than along single coordinate directions. Second, the progression dynamics: the virtual coordinate $\gamma(t)$ initially advances slowly, effectively \emph{waiting} for the physical system to acquire the path, and only settles to the prescribed rate $\dot\gamma \to \nu = 0.5\ \mathrm{rad/s}$ once $\|\phi\|$ has converged. This waiting behavior is not imposed by hand, it is a structural consequence of the coupling between $R$ and $\gamma$ through $\chi(\xi)$ on the gradient-based component, and is what allows the framework to guarantee path invariance without requiring the initial condition to lie on the path.

\subsection{Point convergence on a self-intersecting Lissajous path}

The second scenario exercises the design freedom in $\nu$. The reference path is a Lissajous curve on $\mathrm{SO}(3)$,
$$
R_\mathrm{ref}(\gamma) = \exp\left([\sin(2\gamma),\, \sin(3\gamma),\, \sin(\gamma)]^\wedge\right),
$$
which self-intersects at the identity ($\gamma = 0 \bmod \pi$). Classical GVF constructions via implicit zero-level sets handle self-intersecting paths poorly, because the tangent direction is ambiguous at the crossing point and the vector field degenerates. The SF-GVF resolves this via the augmented state $\gamma$, which uniquely identifies the branch. The system is also required to converge to a specific point on the path, $\gamma^* = 3.5\ \mathrm{rad}$, and remain there. This is expressed through the progression behavior
$$
\nu(\xi) = -k_\nu\,(\gamma - \gamma^*), \quad k_\nu = 0.3,
$$
which drives $\gamma$ to $\gamma^*$ from any initial value. Together with the SF-GVF's convergence in $R$, this yields simultaneous point convergence in both $\gamma$ and $R$. We use $K = 0.5\,I$ and $R(0) = \exp\left([-\pi/3,\, \pi/3,\, 2\pi/3]^\wedge\right)$.

Figure~\ref{fig: sim2} shows the result. The trajectory acquires the path, traverses it in the direction dictated by the sign of $\nu$, and stops at $R_\mathrm{ref}(\gamma^*)$ with $\omega \to 0$. Two structural points are worth highlighting. First, this operational behavior is unavailable in Euclidean SF-GVF constructions where the parametric speed is implicitly bounded away from zero by minimum-velocity constraints; on $\mathrm{SO}(3)$, the admissibility of $\omega = 0$ makes it a natural target for $\nu$. Second, the SF-GVF's non-degeneracy result (Proposition~\ref{prop: no_sing}) predicts exactly this behavior: the guidance output $\chi(\xi) = 0$ if and only if both $\phi = 0$ and $\nu(\xi) = 0$, which occurs precisely at the target configuration $(R_\mathrm{ref}(\gamma^*), \gamma^*)$. The framework degenerates only at the designed stopping condition, not at incidental configurations.

\begin{figure*}
    \centering
    \includegraphics[width=1.95\columnwidth]{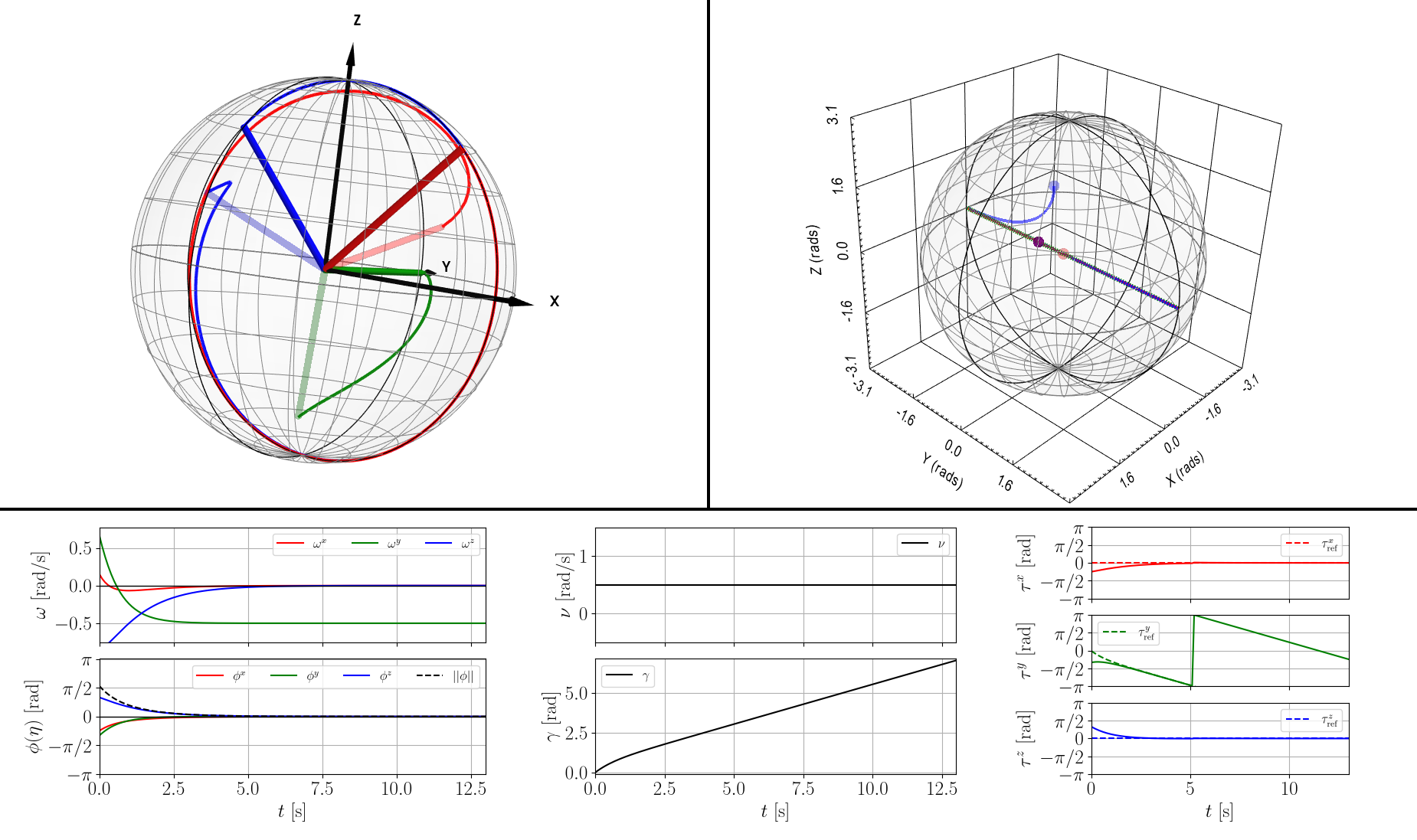}
    \caption{Single-axis rotation with constant progression ($\nu = 0.5\ \mathrm{rad/s}$, $K = 0.8\,I$, $R(0) = \exp([-\pi/4, -\pi/3, \pi/3]^\wedge)$). \textbf{Top left:} attitude $R(t)$ (opaque) and initial $R(0)$ (translucent), with body-axis trajectories on the unit sphere. \textbf{Top right:} $\mathfrak{so}(3)\cong\mathbb{R}^3$ representation via $\tau = \log(R)^\vee$: reference path $\tau_\mathrm{ref}(\gamma)$ (dashed green), reference point $\tau_\mathrm{ref}(\gamma(t))$ (red), physical trajectory $\tau(t)$ (blue). \textbf{Bottom left:} commanded rate $\omega(t)$ (top) and error $\phi(t), \|\phi(t)\|$ (bottom). \textbf{Bottom center:} progression $\nu(t)$ (top) and virtual coordinate $\gamma(t)$ (bottom); $\dot\gamma \to \nu$ after transient. \textbf{Bottom right:} time evolution of $\tau(t)$ (solid) and $\tau_\mathrm{ref}(\gamma(t))$ (dashed).}
    \label{fig: sim1}
\end{figure*}

\begin{figure*}
    \centering
    \includegraphics[width=1.95\columnwidth]{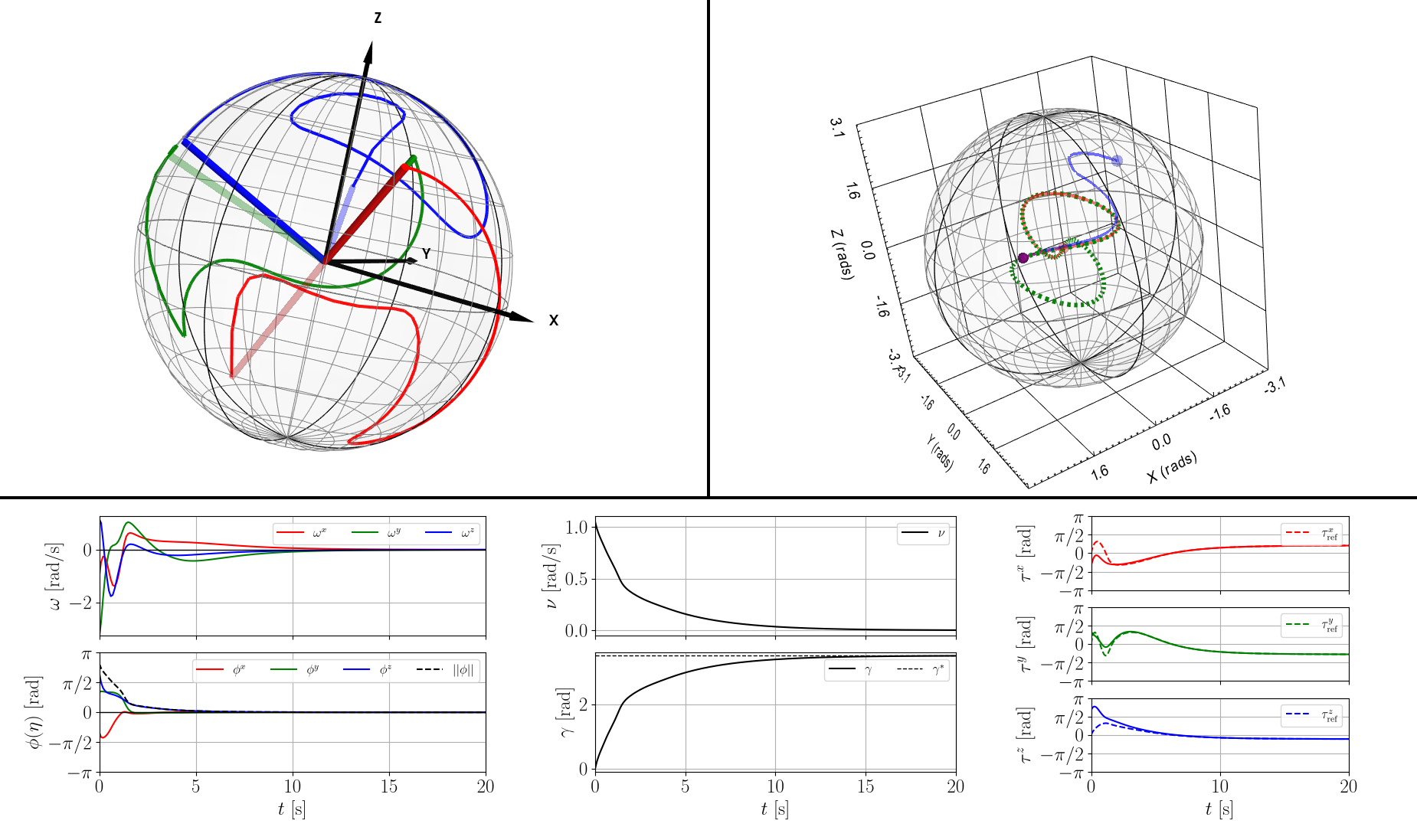}
    \caption{Point convergence on a self-intersecting Lissajous path with state-dependent $\nu(\xi) = -k_\nu(\gamma - \gamma^*)$ ($k_\nu = 0.3$, $\gamma^* = 3.5\ \mathrm{rad}$, $K = 0.5I$, $R(0) = \exp([-\pi/3, \pi/3, 2\pi/3]^\wedge)$). Panels as in Fig.~\ref{fig: sim1}. \textbf{Top right}: the reference $\tau_\mathrm{ref}(\gamma)$ self-intersects at the origin. \textbf{Bottom center}: dashed line marks the target $\gamma^*$; $\gamma \to \gamma^*$ and $\nu \to 0$ at steady state.}
    \label{fig: sim2}
\end{figure*}


\section{Conclusion}
\label{sec: conclusions}

This paper presented a singularity-free guiding vector field (SF-GVF) for path following on $\mathrm{SO}(3)$ by combining an augmented-state construction with the intrinsic geometry of the Lie group. The proposed framework generates body-rate commands directly in $\mathfrak{so}(3)$ without per-step optimization and introduces the progression law $\nu(\xi)$ as an explicit design variable, enabling behaviors such as target-attitude convergence, synchronization, and feasibility-aware path progression. Under the bi-invariant metric ($M=I$), the framework guarantees convergence and non-degeneracy while remaining independent of the specific path or progression design. Simulation results demonstrated both conventional attitude path following and the flexibility provided by state-dependent progression on complex trajectories.

Future work will extend the framework to general left-invariant metrics, allowing the geometry to be shaped according to platform dynamics, and will investigate experimental validation on fixed-wing UAVs using an INDI-based control architecture. Extensions to cooperative multi-agent attitude guidance and hybrid formulations for global feedback on $\mathrm{SO}(3)$ also constitute promising research directions.


\bibliographystyle{elsarticle-num}
\bibliography{biblio}

@article{tee2009barrier,
  author  = {Tee, Keng Peng and Ge, Shuzhi Sam and Tay, Eng Hock},
  title   = {Barrier {L}yapunov functions for the control of output-constrained nonlinear systems},
  journal = {Automatica},
  volume  = {45},
  number  = {4},
  pages   = {918--927},
  year    = {2009},
}

@book{khalil,
  author    = {Khalil, Hassan K.},
  title     = {Nonlinear Systems},
  publisher = {Prentice Hall},
  year      = {2002},
  edition   = {3rd},
  address   = {Upper Saddle River, NJ}
}

@book{krstic1995nonlinear,
  author    = {Krstic, Miroslav and Kokotovic, Petar V. and Kanellakopoulos, Ioannis},
  title     = {Nonlinear and Adaptive Control Design},
  publisher = {John Wiley \& Sons},
  year      = {1995},
  address   = {New York}
}

@book{slotine1991applied,
  author    = {Slotine, Jean-Jacques E. and Li, Weiping},
  title     = {Applied Nonlinear Control},
  publisher = {Prentice Hall},
  year      = {1991},
  address   = {Englewood Cliffs, NJ}
}

@book{beard2012small,
  title={Small unmanned aircraft: Theory and practice},
  author={Beard, Randal W and McLain, Timothy W},
  year={2012},
  publisher={Princeton university press}
}

@article{smeur2016adaptive,
  title={Adaptive incremental nonlinear dynamic inversion for attitude control of micro air vehicles},
  author={Smeur, Ewoud JJ and Chu, Qiping and De Croon, Guido CHE},
  journal={Journal of Guidance, Control, and Dynamics},
  volume={39},
  number={3},
  pages={450--461},
  year={2016},
  publisher={American Institute of Aeronautics and Astronautics}
}

@article{sun2022comparative,
  title={A comparative study of nonlinear mpc and differential-flatness-based control for quadrotor agile flight},
  author={Sun, Sihao and Romero, Angel and Foehn, Philipp and Kaufmann, Elia and Scaramuzza, Davide},
  journal={IEEE Transactions on Robotics},
  volume={38},
  number={6},
  pages={3357--3373},
  year={2022},
  publisher={IEEE}
}

@misc{vinicius2026liegroups,
  title={Constructive Vector Fields for Path Following in Fully-Actuated Systems on Matrix Lie Groups}, 
  author={Felipe Bartelt and Vinicius M. Gonçalves and Luciano C. A. Pimenta},
  year={2026},
  eprint={2602.21450},
  archivePrefix={arXiv},
  primaryClass={cs.RO},
}

@article{yao2023manifolds,
  author={Yao, Weijia and Lin, Bohuan and Anderson, Brian D. O. and Cao, Ming},
  journal={IEEE Transactions on Automatic Control}, 
  title={Topological Analysis of Vector-Field Guided Path Following on Manifolds}, 
  year={2023},
  volume={68},
  number={3},
  pages={1353-1368},
}

@article{yao2022coord,
  title={Guiding vector fields for the distributed motion coordination of mobile robots},
  author={Yao, Weijia and de Marina, H{\'e}ctor Garc{\'\i}a and Sun, Zhiyong and Cao, Ming},
  journal={IEEE Transactions on Robotics},
  volume={39},
  number={2},
  pages={1119--1135},
  year={2022},
  publisher={IEEE}
}

@article{yao2021singularity,
  title={Singularity-free guiding vector field for robot navigation},
  author={Yao, Weijia and de Marina, H{\'e}ctor Garcia and Lin, Bohuan and Cao, Ming},
  journal={IEEE Transactions on Robotics},
  volume={37},
  number={4},
  pages={1206--1221},
  year={2021},
  publisher={IEEE}
}

@article{rezende2021constructive,
  title={Constructive time-varying vector fields for robot navigation},
  author={Rezende, Adriano MC and Goncalves, Vinicius M and Pimenta, Luciano CA},
  journal={IEEE Transactions on Robotics},
  volume={38},
  number={2},
  pages={852--867},
  year={2021},
  publisher={IEEE}
}

@article{kapitanyuk2017guiding,
  title={A guiding vector-field algorithm for path-following control of nonholonomic mobile robots},
  author={Kapitanyuk, Yuri A and Proskurnikov, Anton V and Cao, Ming},
  journal={IEEE Transactions on Control Systems Technology},
  volume={26},
  number={4},
  pages={1372--1385},
  year={2017},
  publisher={IEEE}
}

@article{goncalves2010vector,
  title={Vector fields for robot navigation along time-varying curves in $ n $-dimensions},
  author={Goncalves, Vinicius M and Pimenta, Luciano CA and Maia, Carlos A and Dutra, Bruno CO and Pereira, Guilherme AS},
  journal={IEEE Transactions on Robotics},
  volume={26},
  number={4},
  pages={647--659},
  year={2010},
  publisher={IEEE}
}

@article{aguiar2008performance,
  title={Performance limitations in reference tracking and path following for nonlinear systems},
  author={Aguiar, A Pedro and Hespanha, Joao P and Kokotovi{\'c}, Petar V},
  journal={Automatica},
  volume={44},
  number={3},
  pages={598--610},
  year={2008},
  publisher={Elsevier}
}

@article{maithripala2015intrinsic,
  title = {An intrinsic {PID} controller for mechanical systems on {Lie} groups},
  author={Maithripala, DH Sanjeeva and Berg, Jordan M},
  journal={Automatica},
  volume={54},
  pages={189--200},
  year={2015},
  publisher={Elsevier}
}

@article{lee2010geometric,
  title = {Geometric tracking control of a quadrotor {UAV} on {SE(3)}},
  author={Lee, Taeyoung and Leok, Melvin and McClamroch, N Harris},
  booktitle={49th IEEE conference on decision and control (CDC)},
  pages={5420--5425},
  year={2010},
  organization={IEEE}
}

@article{maithripala2006almost,
  title={Almost-global tracking of simple mechanical systems on a general class of {Lie} groups},
  author={Maithripala, DH Sanjeeva and Berg, Jordan M and Dayawansa, Wijesuriya P},
  journal={IEEE Transactions on Automatic Control},
  volume={51},
  number={2},
  pages={216--225},
  year={2006},
  publisher={IEEE}
}

@article{bullo1999tracking,
  title={Tracking for fully actuated mechanical systems: a geometric framework},
  author={Bullo, Francesco and Murray, Richard M},
  journal={Automatica},
  volume={35},
  number={1},
  pages={17--34},
  year={1999},
  publisher={Elsevier}
}

@misc{so3_catalanes,
title={A micro Lie theory for state estimation in robotics}, 
author={Joan Solà and Jeremie Deray and Dinesh Atchuthan},
year={2021},
}

@book{hairer2006gni,
  author    = {Hairer, Ernst and Lubich, Christian and Wanner, Gerhard},
  title     = {Geometric Numerical Integration},
  edition   = {2},
  publisher = {Springer},
  year      = {2006}
}

@book{Bullo2005,
  author    = {F. Bullo and A. Lewis},
  title     = {Geometric Control of Mechanical Systems},
  volume    = {49},
  publisher = {Springer-Verlag},
  address   = {New York},
  year      = {2005},
  isbn      = {978-0-387-22442-0}
}

@article{bhat2000topological,
  title={A topological obstruction to continuous global stabilization of rotational motion and the unwinding phenomenon},
  author={Bhat, Sanjay P and Bernstein, Dennis S},
  journal={Systems \& control letters},
  volume={39},
  number={1},
  pages={63--70},
  year={2000},
  publisher={Elsevier}
}

@article{wang2021hybrid,
  title={Hybrid feedback for global tracking on matrix lie groups {SO(3)} and {SE(3)}},
  author={Wang, Miaomiao and Tayebi, Abdelhamid},
  journal={IEEE Transactions on Automatic Control},
  volume={67},
  number={6},
  pages={2930--2945},
  year={2021},
  publisher={IEEE}
}

@article{berkane2017hybrid,
  title={Hybrid global exponential stabilization on {SO(3)}},
  author={Berkane, Soulaimane and Abdessameud, Abdelkader and Tayebi, Abdelhamid},
  journal={Automatica},
  volume={81},
  pages={279--285},
  year={2017},
  publisher={Elsevier}
}

@article{mayhew2013synergistic,
  title={Synergistic Hybrid Feedback for Global Rigid-Body Attitude Tracking on {SO(3)}},
  author={Mayhew, Christopher G and Teel, Andrew R},
  journal={IEEE Transactions on Automatic Control},
  volume={58},
  number={11},
  pages={2730--2742},
  year={2013},
  publisher={IEEE}
}

@article{mayhew2011quaternion,
  title={Quaternion-based hybrid control for robust global attitude tracking},
  author={Mayhew, Christopher G and Sanfelice, Ricardo G and Teel, Andrew R},
  journal={IEEE Transactions on Automatic control},
  volume={56},
  number={11},
  pages={2555--2566},
  year={2011},
  publisher={IEEE}
}

@article{fornasier2025equivariant,
  title={Equivariant symmetries for inertial navigation systems},
  author={Fornasier, Alessandro and Ge, Yixiao and van Goor, Pieter and Mahony, Robert and Weiss, Stephan},
  journal={Automatica},
  volume={181},
  pages={112495},
  year={2025},
  publisher={Elsevier}
}

@article{fornasier2023msceq,
  title={Msceqf: A multi state constraint equivariant filter for vision-aided inertial navigation},
  author={Fornasier, Alessandro and van Goor, Pieter and Allak, Eren and Mahony, Robert and Weiss, Stephan},
  journal={IEEE Robotics and Automation Letters},
  volume={9},
  number={1},
  pages={731--738},
  year={2023},
  publisher={IEEE}
}




\end{document}